%% file: Arxiv2022KLv5Otc.tex
\documentclass{article}
\input{preamble}

 \PassOptionsToPackage{numbers, compress}{natbib}

\newcommand{\bx}{\mathbf{x}}
\newcommand{\bX}{\mathbf{X}}
\newcommand{\by}{\mathbf{y}}
\newcommand{\bY}{\mathbf{Y}}

\newcommand{\bPhi}{\bm{\Phi}}
\newcommand{\bQ}{\mathbf{Q}}

\newcommand{\calE}{\mathcal{E}}
\newcommand{\calX}{\mathcal{X}}
\newcommand{\calS}{\mathcal{S}}
\newcommand{\calF}{\mathcal{F}}

\newcommand{\calM}{\mathcal{M}}

\newcommand{\calG}{\mathcal{G}} 

\newcommand{\bbR}{\mathbb{R}} 
\newcommand{\bbZ}{\mathbb{Z}} 
\newcommand{\bbE}{\mathbb{E}} 
\newcommand{\bbP}{\mathbb{P}} 

     \usepackage[preprint]{neurips_2022}

\usepackage[utf8]{inputenc} 
\usepackage[T1]{fontenc}    
\usepackage{hyperref}       
\usepackage{url}            
\usepackage{booktabs}       
\usepackage{amsfonts}       
\usepackage{nicefrac}       
\usepackage{microtype}      
\usepackage{xcolor}         

\title{Generalization Bounds on Multi-Kernel Learning with Mixed Datasets}

\author{%
	Lan V. Truong\thanks{Use footnote for providing further information
		about author (webpage, alternative address)---\emph{not} for acknowledging
		funding agencies.} \\
	Department of Engineering\\
	University of Cambridge\\
	Cambridge, CB2 1PZ \\
	\texttt{lt407@cam.ac.uk} \\
}
\begin{document}
	\maketitle
		\begin{abstract}
		This paper presents novel generalization bounds for the multi-kernel learning problem. Motivated by applications in sensor networks and spatial-temporal models, we assume that the dataset is mixed where each sample is taken from a finite pool of Markov chains. Our bounds for learning kernels admit $O(\sqrt{\log m})$ dependency on the number of base kernels and $O(1/\sqrt{n})$ dependency on the number of training samples. However, some $O(1/\sqrt{n})$ terms are added to compensate for the dependency among samples compared with existing generalization bounds for multi-kernel learning with i.i.d. datasets. 
	\end{abstract}
	\section{Introduction}
	Kernel methods are widely used in statistical learning, which use kernel functions to operate in a  high-dimensional implicit feature space without ever computing the coordinates of the data in that space. The best known member is Support Vector Machines (SVMs) for classification and regression. The performance of a  kernel machine depends on the data representation via the choice of kernel function. Rather than requesting the user to commit to a specific kernel, which may not be optimal due to the user's limited knowledge about the task, learning kernel methods require the user only to supply a family of kernels. The learning algorithm then selects both the specific kernel out of that family, and the hypothesis defined based on that kernel. Kernel learning can range from the width parameter selection of Gaussian kernels to obtaining an optimal linear combination from a set of finite candidate kernels. The later is often referred to as multiple kernel learning (MKL) in machine learning. 
	
	Lanckriet et al. \citep{Lanckriet2004a} pioneered work on MKL and proposed a semi-definite programming (SDP) approach to automatically learn a linear combination of candidate kernels for the case of SVMs. There is a large body of literature dealing with various aspects of the problem of learning kernels, including theoretical questions, optimization problems related
	to this problem, and experimental results. \citep{YimingCOLT2009} developed a probabilistic generalization bound for learning the kernel problem via Rademacher chaos complexity. They also showed how to estimate the empirical Rademacher chaos complexity by well-established metric entropy integrals and pseudo-dimension of the set of the candidate kernels. For a convex combination of $m$ kernels, their bounds is in $O(\sqrt{m})$. \citep{Cortes2010} improved Yiming and Campbell's bound to $O(\sqrt{\log m})$ by using the generalization bounds for classifiers in \citep{Koltchinskii2002} and better bounding the Rademacher complexity function via combinatorial tools. They also presented other bounds for learning with a non-negative
	combination of $m$ base kernels with an $L_q$ regularization for other value of $q$.  \citep{Hussain2011} presented a new Rademacher complexity bound which is additive in the (logarithmic) kernel complexity and margin term. This independence is superior to all previously Rademacher bounds for learning a convex combination of kernels, including \citep{Cortes2010}. \citep{Liu2017InfiniteKL} proposed a new kernel learning method based which can learn the optimal kernel with sharp generalization bounds over the convex hull of a possibly infinite set of basic kernels. Some other works have been focusing on designing algorithms to select optimal kernels for various models in practice \citep{Akian2022Learning,Lalchand2022KernelLF}. 
	
	In the above research literature, the dataset is usually assumed to be generated by an i.i.d. process with unknown distribution. However, in many applications in machine learning such as speech, handwriting, gesture recognition, and bio-informatics, the samples of data are
	usually correlated. \citep{Truong2022GEB} has recently provided generalization bounds for learning with Markov dataset based on Rademacher and Gaussian complexity functions. In this work, we develop a novel generalization bound for MKL based on Rademacher complexity function for mixed datasets where each sample is selected from a finite pool of Markov chains. Our problem setting is motivated by 
	the fact that the data can be a mixture of many populations (sources) where each the data in each source is correlated in time such as in time-series and spatio-temporal datasets. In spatio-temporal datasets, samples are usually correlated in both time and space domains. However, in this work we assume that the data is uncorrelated in the space domain which can happen in many applications such as in sensor networks where the center has data from different sensors (populations). Our work can be also considered a step toward understanding the effects of dataset structures on the generalization errors in machine learning. 
	\section{Preliminaries} \label{sec1}
	\subsection{Problem settings} \label{sub:setting}
	In this paper, we use the same problem setting as \citep{Hussain2011} except for the mixed dataset assumption. Let $[n]=\{1,2,\cdots, n\}$ and we are interested in the classification problem on the input space $\calX \subset \bbR^d$ and output space $\calY=\{\pm 1\}$. The relationship between input $X$ and output $Y$ is specified by a set of training samples $\bz=\{(X_i,Y_i): X_i \in \calX, Y_i \in \calY, i \in [n]\}$.

	Let $\calK$ be a prescribed (possible infinite) set of candidate (base kernels) and denote the candidate reproducing kernel Hilbert space (RKHS) with kernel $K$ by $\calH_K$ with norm $\|\cdot\|_K$. For any kernel function $K$, we denote by $\bPhi_K: x \mapsto \calH_K$ the feature mapping from $\calX$ to the reproducing kernel Hilbert space $\calH_K$ induced by $K$. As \citep{Bartlett2002jmlr}, we limit 
	\begin{align}
	\calH_K:=\bigg\{x\mapsto \langle \bw, \bPhi_K(x)\rangle_K: \|\bw\|_K \leq B \bigg\},
	\end{align} for some positive finite constant $B>0$ such as in the support vector machine (SVM). In addition, we always assume that the quantity $\kappa:=\sup_{K \in \calK, x \in \calX} \sqrt{K(x,x)}$ is finite. 
	
	In research literature, the set of kernels $\calK$ is usually a non-negative combinations of a finite set of base kernels, say $\{K_1,K_2,\cdots,K_m\}$,  with the mixture of weights obeying an $L_q$ constraints (cf.~\citep{Cortes2010}):
	\begin{align}
	&\calK_m^q(K_1,K_2,\cdots,K_m)\nn\\
	&\qquad :=\bigg\{K=\sum_{i=1}^m \eta_i K_i: \eta \geq 0, \sum_{i=1}^m \eta_i^q=1 \bigg\}.
	\end{align}
	
	The MKL can be described as finding a function $f$ from the class of functions  $\calH_{\calK}=\bigcup_{K \in \calK}\calH_K$ that minimizes
	\begin{align}
	\calE^{\phi}_{\bz}(f):=\frac{1}{n}\sum_{i=1}^n \phi(Y_i f(X_i)/\delta),
	\end{align} where $\phi(t)=(1-t)_+$ which is the hinge loss. We call $\delta \in (0,1]$ the margin. 
	
	In this work, we assume that features $\{X_n\}_{n=1}^{\infty}$ are generated by a finite pool of order-$1$ Markov chains $\calP$ on $\Lambda$ (a mixed dataset)\footnote{Extension to high-order Markov chains is obtained based on the conversion of these Markov chains to equivalent $1$-order Markov chains (see~\cite{Truong2022GEB}).}. Furthermore, the probability that each sample $X_n$ is from the Markov chain $P$ is $\mu_P$ for each $P \in \calP$. Besides, for each sub-sequence $\{X_{P,k}\}_{k=1}^{\infty}$ of $\{X_n\}_{n=1}^{\infty}$, which is a Markov chain with stochastic matrix $P \in \calP$, their corresponding labels, i.e., $\{Y_{P,k}\}_{k=1}^{\infty}$, are generated by $\{X_{P,k}\}_{k=1}^{\infty}$ via a Hidden Markov Model (HMM) with emission probability $P_{P}(y|x)$. With this assumption, $v_P:=\{(X_{P,k},Y_{P,k})\}_{k=1}^{\infty}$ forms a Markov chain for each fixed $P \in \calP$ with stationary distribution $\pi_P$ \citep{Truong2022GEB}. The i.i.d. and Markov datasets can be considered as special cases of this dataset structure. 
	
	The true error or \emph{generalization error} of a function $f$ is defined as:
	\begin{align}
	R(f):=\sum_{P \in \calP} \mu_P \bbP_{(X,Y)\sim \pi_P}[Yf(X)\leq 0],
	\end{align} and the empirical margin error of $f$ with margin margin $\delta \in (0,1]$:
	\begin{align}
	\hatR_{\delta}(f)=\frac{1}{n}\sum_{i=1}^n \bone_{Y_i f(X_i)<\delta},	
	\end{align} where $\bone$ is the indicator function. The estimation error $\calE_{\delta}(f)$ is defined as
	\begin{align}
	\calE_{\delta}(f):=R(f)- \hatR_{\delta}(f).
	\end{align} 
	Our target is to find an upper PAC-bound on $\calE_{\delta}(f)$ which holds for any $f \in \calH_{\calK}$.
	\subsection{Mathematical Backgrounds}\label{sec:background}
	Let a Markov chain $\{X_n\}_{n=1}^{\infty}$ on a state space $\calS$ with transition kernel $Q(x,dy)$ and the initial state $X_1 \sim \nu$, where $\calS$ is a Polish space in $\bbR$. In this paper, we consider Markov chains which are irreducible and positive-recurrent, so the existence of a stationary distribution $\pi$ is guaranteed. An irreducible and recurrent Markov chain on an infinite state-space is called Harris chain \citep{TR1979}. A Markov chain is called \emph{reversible} if the following detailed balance condition is satisfied:
	\begin{align}
	\pi(dx)Q(x,dy)=\pi(dy)Q(y,dx),\qquad \forall x, y \in \calS.
	\end{align} 
	Define
	\begin{align}
	d(t)=\sup_{x \in \calS} d_{\rm{TV}}(Q^t(x,\cdot),\pi)
	\end{align}
	and
	\begin{align}
	t_{\rm{mix}}(\eps):=\min\{t: d(t)\leq \eps\},
	\end{align}
	and
	\begin{align}
	\tau_{\min}:=\inf_{0\leq \eps\leq 1}t_{\rm{mix}}(\eps)\bigg(\frac{2-\eps}{1-\eps}\bigg)^2,\quad 
	t_{\rm{mix}}:=t_{\rm{mix}}(1/4) \label{deftaumin}.
	\end{align}
	
	Let $L_2(\pi)$ be the Hilbert space of complex valued measurable functions on $\calS$ that are square integrable w.r.t. $\pi$. We endow $L_2(\pi)$ with inner product $\langle f,g \rangle:= \int f g^* d\pi$, and norm $\|f\|_{2,\pi}:=\langle f, f\rangle_{\pi}^{1/2}$. $Q$ can be viewed as a linear operator (infinitesimal generator) on $L_2(\pi)$, denoted by $\bQ$, defined as $(\bQ f)(x):=\bbE_{Q(x,\cdot)}(f)$, and the reversibility is equivalent to the self-adjointness of $\bQ$. The operator $\bQ$ acts on measures on the left, creating a measure $\mu \bQ$, that is, for every measurable subset $A$ of $\calS$, $\mu \bQ (A):=\int_{x \in \calS} Q(x,A)\mu(dx)$. Let $\bE_{\pi}$ be the associated averaging operator defined by $(E_{\pi})(x,y)=\pi(y), \forall x,y \in \calS$, and 
	\begin{align}
	\lambda=\|\bQ-\bE_{\pi}\|_{L_2(\pi)\to L_2(\pi)} \label{defL2gap},
	\end{align} where
	$
	\|B\|_{L_2(\pi)\to L_2(\pi)}=\max_{v: \|v\|_{2,\pi}=1}\|Bv\|_{2,\pi}.
	$
	For a Markov chain with stationary distribution $\pi$, we define the \emph{spectrum} of the chain as
	\begin{align}
	S_2:=\big\{\xi \in \bbC: (\xi \bI-\bQ) \enspace \mbox{is not invertible on} \enspace L_2(\pi)\big\}.
	\end{align}
	It is known that $\lambda=1-\gamma^*$ \cite{Daniel2015}, 
	where
	\begin{align}
	\gamma^*&:=\begin{cases} 1-\sup\{|\xi|: \xi \in \calS_2, \xi \neq 1\},\nn\\
	\qquad \qquad \mbox{if eigenvalue $1$ has multiplicity $1$,}\\
	0,\qquad\qquad \mbox{otherwise}\end{cases}
	\end{align} is the \emph{the absolute spectral gap} of the Markov chain. The absolute spectral gap can be bounded by the mixing time $t_{\rm{mix}}$ of the Markov chain by the following expression:
	\begin{align}
	\bigg(\frac{1}{\gamma^*}-1\bigg)\log 2 \leq t_{\rm{mix}} \leq \frac{\log (4/\pi_*)}{\gamma_*},	
	\end{align}
	where $\pi_*=\min_{x\in \calS} \pi_x$ is the \emph{minimum stationary probability}, which is positive if $Q^k>0$ (entry-wise positive) for some $k\geq 1$. See \citep{WK19ALT} for more detailed discussions. In \citep{Combes2019EE, WK19ALT}, the authors provided algorithms to estimate $t_{\rm{mix}}$ and $\gamma^*$ from a single trajectory. 
	
	For a Markov chain with transition kernel $Q(x,dy)$, and stationary distribution $\pi$, we define the time reversal of $Q$ as the Markov kernel
	\begin{align}
	Q^*(x,dy):=\frac{Q(y,dx)}{\pi(dx)} \pi(dy).
	\end{align}
	Then, the linear operator $\bQ^*$ is the adjoint of the linear operator $\bQ$, on $L_2(\pi)$ if the Markov chain is reversible. For reversible chains, $S_2$ lies on the real line. We define the \emph{spectral gap} for reversible chains as
	\begin{align}
	\gamma&:=\begin{cases} 1-\sup\{\xi: \xi \in \calS_2, \xi \neq 1\},\nn\\
	\qquad \qquad \mbox{if eigenvalue $1$ has multiplicity $1$,}\\
	0, \qquad \qquad \mbox{otherwise}\end{cases}.
	\end{align}
	Obviously, $\gamma\geq \gamma^*$. For non-reversible Markov chain, we define a new quantity, called the \emph{pseudo spectral gap} of $\bQ$, as
	\begin{align}
	\gamma_{\rm{ps}}:=\max_{k\geq 1} \big\{\gamma((\bQ^*)^k \bQ^k)/k\big\} \label{defps},
	\end{align} where $\gamma((\bQ^*)^k \bQ^k)$ denotes the spectral gap of the self-adjoint operator $(\bQ^*)^k \bQ^k$. It is known that the pseudo-spectral gap $\gamma_{\rm{ps}}$ and the mixing time $t_{\rm{mix}}$ of an ergodic (irreducible) and reversible Markov chain is related to each other \citep{Daniel2015}. In \citep{WK19ALT}, the authors also provided algorithms to estimate $t_{\rm{mix}}$ and $\gamma_{\rm{ps}}$ from a single trajectory.
	\subsubsection{Reproducing Kernel Hilbert Space}
	Let $\calF$ be a class of functions defined in $E$, forming a Hilbert space (complex or real). The function $K(x,y)$ of $x$ and $y$ in $E$ is called a \emph{reproducing kernel} of $\calF$ if
	\begin{itemize}
		\item For every $y$, $K(x,y)$ as function of $x$ belongs to $\calF$.
		\item The \emph{reproducing property:} for every $y\in E$ and every $f \in \calF$, 
		\begin{align}
		f(y)=\langle f(x), K(x,y)\rangle_x
		\end{align}
		The subscript $x$ by the scalar product indicates that the scalar product applies to functions of $x$.
	\end{itemize}
	For the existence of a reproducing kernel $K(x,y)$ it is necessary and sufficient that for every $y$ of the set $E$, $f(y)$ be a continuous functional of $f$ running through the Hilbert space $\calF$. $K(x,y)$ is a positive matrix in the sense of Moore, that is, the quadratic form in $\xi_1,\xi_2,\cdots, \xi_n$,
	\begin{align}
	\sum_{i,j=1}^n K(y_i,y_j)\xi_i^*\xi_j \label{eq1}
	\end{align} is non-negative for all $y_1,y_2,\cdots, y_n$ in $E$. This is clear since \eqref{eq1} equals $\|\sum_{i=1}^n K(x,y_i)\xi_i\|^2$, following the reproducing property. The result in \eqref{eq1} admits a converse due essentially to Moorse: \emph{to every positive matrix $K(x,y)$ there corresponds one and only one class of functions with a unique determined quadratic form in it, forming a Hilbert space and admitting $K(x,y)$ as a reproducing kernel.} This class of functions is generated by all the functions of the form $\sum_k \alpha_k K(x,y_k)$. The norm of this function is defined by the quadratic form $\|\sum_k \alpha_k K(x,y_k)\|^2=\sum \sum K(y_i,y_j)\xi_i^*\xi_j$. Refer to \citep{Aronszajn} for more properties of reproducing kernels.
	\subsection{Notations}
	Consider a sequence $\{X_n\}_{n=1}^{\infty}$ on $\Lambda_1 \times \Lambda_2 \times \cdots \times \Lambda_n \times \cdots$ where each sample is taken from a finite pool of Markov chains $\calP$. Assume that the Markov chain $P$ has the stationary distribution $\pi_P$ for each $P \in \calP$. Let $\calF$ be classes of functions from $\calX \to \bbR$. For each function $f \in \calF$, define
	\begin{align}
	P_nf:=\frac{1}{n}\sum_{i=1}^nf(X_i) \label{a3b},
	\end{align} 
	and
	\begin{align}
	Pf:= \sum_{P \in\calP} \mu_P  \int_{\calX} f(x) \pi_P(x) dx \label{a3}.
	\end{align} 
	\emph{The Rademacher complexity function} of the class $\calF$ is defined as
	\begin{align}
	R_n(\calF):=\bbE\bigg[\sup_{f \in \calF} \bigg|n^{-1}\sum_{i=1}^n \eps_i f(X_i)\bigg|\bigg] \label{defRM},
	\end{align} where $\{\eps_i\}$ is a sequence of i.i.d. Rademacher (taking values $+1$ and $-1$ with probability $1/2$ each) random variables, independent of $\{X_i\}$. 
	\section{Main Results}
	In order to obtain generalization error bounds for kernel learning with mixed datasets, we need to develop a new concentration bound and a symmetrization inequality for this type of dataset. First, we introduce how to use the Marton coupling for deriving the McDiarmid's inequality for the mixed dataset. 
	\begin{definition}\citep{Daniel2015}\label{defMarton} Let $\bX:=(X_1,X_2,\cdots,X_n)$ be a vector of random variables taking values in $\Lambda:=\Lambda_1 \times \Lambda_2 \cdots \times \Lambda_n$. We define a Marton coupling for $\bX$ as a set of couplings
		\begin{align}
		\big(\bX^{(x_1,x_2,\cdots,x_i,x_i')}, \barX^{(x_1,x_2,\cdots,x_i,x_i')}\big) \in \Lambda \times \Lambda,
		\end{align} for every $i \in [n]$, every $x_1 \in \Lambda_1, x_2 \in \Lambda_2,\cdots, x_i \in \Lambda_i, \barx_i \in \Lambda_i$, satisfying the following conditions:
		\begin{itemize}
			\item $X_1^{(x_1,x_2,\cdots,x_i,\barx_i)}=x_1, \cdots, X_{i-1}^{(x_1,x_2,\cdots,x_i,\barx_i)}=x_{i-1}, X_i^{(x_1,x_2,\cdots,x_i,x_i')}=x_i$,\\
			$\barX_1^{(x_1,x_2,\cdots,x_i,\barx_i)}=x_1, \cdots, \barX_{i-1}^{(x_1,x_2,\cdots,x_i,\barx_i)}=x_{i-1}, \barX_i^{(x_1,x_2,\cdots,x_i,\barx_i)}=\barx_i$.
			\item $\big(X_{i+1}^{(x_1,x_2,\cdots,x_i,\barx_i)},\cdots, X_n^{(x_1,x_2,\cdots,x_i,\barx_i)}\big)\nn\\
			\sim \calL\big(X_{i+1},X_{i+2},\cdots, X_n|X_1=x_1,\cdots,X_{i-1}=x_{i-1},X_i=x_i\big)$\\
			$\big(\barX_{i+1}^{(x_1,x_2,\cdots,x_,\barx_i)},\cdots, \barX_n^{(x_1,x_2,\cdots,x_i,\barx_i)}\big)\nn\\
			\sim \calL\big(X_{i+1},X_{i+2},\cdots, X_n|X_1=x_1,\cdots,X_{i-1}=x_{i-1}, X_i=\barx_i\big)$
			\item If $x_i=\barx_i$, then $X^{(x_1,x_2,\cdots,x_i,\barx_i)}=\barX^{(x_1,x_2,\cdots,x_i,\barx_i)}$. 
		\end{itemize}
		For a Marton coupling, we define the mixing matrix $\Gamma:=(\Gamma_{i,j})_{i,j\leq n}$ as an upper bound diagonal matrix with $\Gamma_{i,i}:=1$ for all $i\leq n$, and
		\begin{align}
		\Gamma_{j,i}&:=0,\\ \Gamma_{i,j}&:=\sup_{x_1,x_2,\cdots,x_i,\barx_i}\bbP\big[X_j^{(x_1,x_2,\cdots,x_i,\barx_i)} \neq \barX_j^{(x_1,x_2,\cdots,x_i,\barx_i)}\big]
		\end{align} for all $1\leq i<j\leq n$.
	\end{definition} 	
	\begin{definition}
		A partition of a set $\calS$ is the division of $\calS$ into disjoint non-empty subsets that together cover $\calS$. Analogously, we say that $\hat{\bX}=(\hatX_1,\hatX_2,\cdots,\hatX_m)$ is a partition of a vector of random variables $\bX=(X_1,X_2,\cdots,X_n)$ if $(\hatX_i)_{1\leq i\leq m}$ is a partition of the set $\{X_1,X_2,\cdots,X_n\}$. For a partition, we denote the number of elements of $\hatX_i$ by $s(\hatX_i)$ and call $s(\hatX):=\max_{i \in m} s(\hatX_i)$ the size of the partition.  
	\end{definition}
	Then, the following result can be shown.
	\begin{lemma} \label{ma:lem} Let $\eps \in [0,1]$. Suppose that $\bX:=(X_1,X_2,\cdots,X_n)$ is a mixed Markov sequence with transition probabilities in a set $\calP$. Then, there exists a partition $\hat{\bX}$ of $\bX$ and a Marton coupling for this partition $\hat{\bX}$ whose mixing matrix $\Gamma$ satisfies
		\begin{align}
		\Gamma:=\diag(\Gamma_P: P \in \calP) \label{uform}
		\end{align} where
		\begin{align}
		\Gamma_P:\leq \begin{bmatrix}
		1&1&\eps&\eps^2&\eps^3&\cdots &\eps^{\mu_P n-3}&\eps^{\mu_P n-2}\\0&1&1&\eps&\eps^2&\cdots &\eps^{\mu_P n-4}&\eps^{\mu_P n-3}\\ \vdots&\vdots&\vdots&\vdots&\vdots&\cdots&\vdots&\vdots\\ 0&0&0&0&\cdots&\cdots&\cdots& 1
		\end{bmatrix} \label{G1}.
		\end{align}
		Here, $A \leq B$ if each element in the matrix $A$ is less than or equal to the corresponding element (i.e., the same row and column) in the matrix $B$. 
	\end{lemma}
	\begin{proof}
		Let 
		\begin{align}
		\calT_P:= \big\{i: X_i \enspace \mbox{is taken from the Markov chain $P$}\big\}
		\end{align} for all $P \in \calP$. Then, we form a Marton coupling for the mixed dataset as follows. We partition the sequence $\bX=(X_1,X_2,\cdots,X_n)$ into $|\calP|$ partition $\calT_P$'s. In each partition $\calT_P$, we use the same sub-partition as in \cite[Lemma 2.4]{Daniel2015}. The Marton coupling for each mixed sequence $\bX=(X_1,X_2,\cdots,X_n)$ is defined as
		\begin{align}
		&\calL\big(X_{i+1},X_{i+2},\cdots, X_n|X_1=x_1,\cdots,X_i=x_i\big)\nn\\
		&\quad= \prod_{P\in \calP}\calL_P^*\big(X_j:j \geq i+1, j \in \calT_P\big|x_j:j\leq i, j \in \calT_P\big),  
		\end{align} where $\calL_P^*\big(X_j:j \geq i+1, j \in \calT_P\big|x_j:j\leq i, j \in \calT_P\big)$ is the optimal law in \cite[Lemma 2.4]{Daniel2015} for each Markov chain $P \in \calP$. 
		
		By the partition and Marton coupling, $\Gamma$ has the form \eqref{uform}.
	\end{proof} 
	We also recall the following result.
	\begin{lemma}\label{danie:lem} Let $\bX=(X_1,X_2,\cdots,X_n)$ be a sequence of random variables, $\bx \in \Lambda, \bx \sim P$. Let $\hat{\bX}=(\hatX_1,\hatX_2,\cdots, \hatX_m)$ be a partition of this sequence, $\hatX\in \hat{\Lambda}, \hat{\bX} \sim \hatP$. Suppose that we have a Marton coupling for $\hat{\bX}$ with matrix $\Gamma$. Let $c \in \bbR_+^n$, and define $C(c)\in \bbR_+^n $ as
		\begin{align}
		C_i(c):=\sum_{j\in \calI(\hatX_i)} c_j
		\end{align} for $i\leq m$.
		If $f:\Lambda \to \bbR$ is such that
		\begin{align}
		f(\bx)-f(\by)\leq c\sum_{i=1}^n  \bone\{x_i\neq y_i\}
		\end{align} for every $\bx,\by \in \Lambda$, then for any $\lambda \in \bbR$, we have
		\begin{align}
		\bbP\bigg(\big|f(\bX)-\bbE[f(\bX)]\big|\geq t\bigg)\leq 2 \exp\bigg(-\frac{2t^2}{\|\Gamma C(c)\|^2}\bigg).
		\end{align}
	\end{lemma}
	
	Now, we introduce a modified version of McDiarmid's inequality for the mixed Markov chain, which extends the McDiardmid's inequality for Markov chain in \citep{Daniel2015}, whose proof is based on Lemma \ref{ma:lem} and Lemma \ref{danie:lem}.
	\begin{lemma} \label{thm6key} Let $X_1,X_2,\cdots,X_n$ be a mixed sequence of random variable on $\Lambda:=\underbrace{\Lambda_1 \times \Lambda_2  \times \cdots \times \Lambda_n}_{n \enspace \mbox{times}}$ with the transition probability sequence $P_{i-1,i}(\cdot,\cdot), i \in [n]$. Assume that $P_{i-1,i} \in \calP$ where $\calP$ is the pool of Markov chains. Assume that the mixing time of the Markov segment $P \in \calP$ is $\tau_P(\eps)$ for any $0\leq \eps\leq 1$. Define
		\begin{align}
		\tau_{\min,P}:=\inf_{0\leq \eps\leq 1} \tau_P(\eps)\bigg(\frac{2-\eps}{1-\eps}\bigg)^2,
		\end{align} 
		and
		\begin{align}
		\tau_{\min}:= \bigg(\sum_{P\in \calP} \sqrt{\mu_P\tau_{\min,P}}\bigg)^2.
		\end{align}
		Suppose $f: \Lambda \to \bbR$ such that
		\begin{align}
		f(\bx)-f(\by)\leq c \sum_{i=1}^n  \bone\{x_i\neq y_i\}
		\end{align} for every $\bx,\by \in \Lambda$. Then, for any $t\geq 0$, it holds that
		\begin{align}
		\bbP\bigg(\big|f(\bX)- \bbE[f(\bX)]\big|\geq t\bigg)\leq 2\exp\bigg(-\frac{2t^2}{c^2 n \tau_{\min}}\bigg) \label{asque}.
		\end{align}
	\end{lemma}
	\begin{proof} 
		For the mixed dataset setting, we have
		\begin{align}
		\Gamma C(c)&= \sum_{P\in \calP} \Gamma C(c)_P\\
		&\leq \sum_{P \in \calP} \|c\|_P \sqrt{\tau_{\min,P}}\\
		&= c \sum_{P \in \calP} \sqrt{n  \mu_P \tau_{\min,P}}\\
		&= c \sqrt{n \tau_{\min}}.
		\end{align}
	\end{proof}
	A variant of Lemma \ref{thm6key} for both revertible and non-revertible Markov chains may be developed based on the spectral method in functional analysis. See Section \ref{beistein:sm} in the Supplement Material for our development of a new Beinstein inequality based on this method and our introduction of our generalized concept ``aggregated pseudo spectral gap". 
	
	Next, the following symmetrization inequality can be proved based on \cite[Lemma 1]{Truong2022GEB}. See Appendix \ref{lem:keymodxproof} for a proof for this fact.
	\begin{lemma}\label{lem:keymodx} Let $\calF$ be a class of functions such that $\|f\|_{\infty}\leq M$ for some $M\in \bbR_+$. Define
		\begin{align}
		A_n:=\max_{P \in \calP}\sqrt{\frac{2M}{n(1-\lambda_P)}+ \frac{64 M^2}{n^2(1-\lambda_P)^2} \bigg\|\frac{dv}{d\pi}-1\bigg\|_2} \label{defAtnmod},
		\end{align}	 where $\lambda_P:=1-\gamma_P^*$ and $\gamma_P^*$ is the absolute spectral-gap of the Markov segment $P \in \calP$. 
		Then, for all $n\in \bbZ^+$, the following holds:
		\begin{align}
		\bbE\big[\big\|P_n-P\big\|_{\calF}\big]  \leq 2\bbE\big[\|P_n^0\|_{\calF}\big]+A_n   \label{F1eqmodx},
		\end{align} where $\|P_n^0\|_{\calF}:=\sup_{f \in \calF} \big|\frac{1}{n}\sum_{i=1}^n \eps_i f(X_i)\big|$.
	\end{lemma}
	The following generalization bound is an extension of \citep[Theorem 2]{Truong2022GEB}. See a detailed proof in the Supplement Material.
	\begin{proposition} \label{akey} Recall the definition of the mixed sequence $X_1,X_2,\cdots,X_n$ in Section \ref{sec:background}. Assume that the Markov chain segment $v_P$ has the stationary distribution $\pi_P$ for all $P \in \calP$ and $X_1 \sim \nu$ for some probability measure $\nu$ in $\calS$ such that 
		$\nu<< \pi_P$ for all $P \in \calP$. Let $\varphi$ is a non-increasing function such that $\varphi(x)\geq \bone_{(-\infty,0]}$ for all $x \in \bbR$. For any $n \in \bbZ_+$, define
		\begin{align}
		B_n:=\max_{P \in \calP}\sqrt{\frac{2}{n(1-\lambda_P)}+ \frac{64 }{n^2(1-\lambda_P)^2} \bigg\|\frac{dv}{d\pi}-1\bigg\|_2} \label{defBtnmod}.
		\end{align}
		Then, for any $t>0$, 
		\begin{align}
		&\bbP\bigg(\exists f \in \calF: P\{f\leq 0\}\nn\\
		&\qquad > \inf_{\delta \in (0,1]}\bigg[P_n\varphi\bigg(\frac{f}{\delta}\bigg)+\frac{8L(\varphi)}{\delta}R_n(\calF)\nn\\
		&\qquad \qquad  + \bigg(t+\log \log 2 \delta^{-1}\bigg)\sqrt{\frac{\tau_{\min}}{n}}+B_n\bigg]\bigg)\nn\\
		&\qquad \qquad \leq \frac{\pi^2}{3}\exp\big(-2t^2\big).
		\end{align}
		Especially, with probability at least $1-\alpha$, it holds for any $\delta \in (0,1]$ that
		\begin{align}
		&P\{f\leq 0\}\leq \inf_{\delta \in (0,1]}\bigg[P_n\varphi\bigg(\frac{f}{\delta}\bigg)+\frac{8L(\varphi)}{\delta}R_n(\calF) \nn\\
		&\qquad + \bigg(\sqrt{\frac{1}{2}\ln \frac{\pi^2}{3\alpha}}+\log \log 2 \delta^{-1}\bigg)\sqrt{\frac{\tau_{\min}}{n}}+B_n\bigg].
		\end{align}
	\end{proposition}
	In addition, by combining \citep[Theorem 7]{Hussain2011} and \citep[Lemma 22]{Bartlett2002jmlr}, the following bound on the Rademacher complexity function is achieved.
	\begin{lemma} \label{asto1} For any $\alpha \in (0,1)$, with probability at least $1-\alpha/2$, the empirical Rademacher complexity $R_n(\calH_{\calK})$ of the class $R_n(\calH_{\calK})$ satisfies
		\begin{align}
		R_n(\calH_{\calK})\leq  \frac{2B\kappa}{\sqrt{n}}+ 8B\kappa \sqrt{\frac{\log(2(m+1)/\alpha)}{2n}}. 
		\end{align}
	\end{lemma} 
	
	From Proposition \ref{akey} and  Lemma \ref{asto1}, a novel bound on the estimation error of MKL algorithms for the mixed dataset is derived.
	\begin{theorem} \label{thm1} Then, for any $\alpha\in (0,1)$, with probability at least $1-\alpha$, it holds that 
		\begin{align}
		&\calE_{\delta}(f)\leq \frac{8}{\delta} \bigg(\frac{2B\kappa}{\sqrt{n}}+ 8B\kappa \sqrt{\frac{\log(2(m+1)/\alpha)}{2n}}\bigg)\nn\\
		&\quad \quad +  \bigg(\sqrt{\frac{1}{2}\ln \frac{2\pi^2}{3\alpha}}+\log \log 2 \delta^{-1}\bigg)\sqrt{\frac{\tau_{\min}}{n}}+B_n \label{c}
		\end{align} for any $f \in \calK, \delta \in (0,1]$  and $m>1$.
	\end{theorem}
	The bound \eqref{c} admits $O(\sqrt{\log m})$ dependency on the number of base kernels and $O(1/\sqrt{n})$ dependency on the number of training samples as the best generalization bounds for multi-kernel learning with i.i.d. datasets \citep{Hussain2011}. The term $B_n$ represents the effect of data structures on the generalization error (see \citep{Truong2022GEB} for detailed discussions). 
	\begin{proof}[Proof of Theorem \ref{thm1}]
		By Cauchy-Schwartz inequality,  for all $f \in \calH_{\calK}$, there exists some $K \in \calK$ such that
		\begin{align}
		|f(x)| &=\langle \bw, \bPhi(x)\rangle_K\\
		&\leq \|\bw\|_K \|\bPhi(x)\|_K\\
		&= \|\bw\|_K \sqrt{ K(x,x)}\\
		&\leq B \kappa, \quad \forall x \in \calX.
		\end{align}
		Hence, it holds that
		\begin{align}
		\|f\|_{\infty} \leq B \kappa \qquad \forall f \in \calH_{\calK} \label{eq36}.
		\end{align} 
		On the other hand, since each sub-sequence $(X_{P,1},Y_{P,1})-(X_{P,2},Y_{P,2}),\cdots, (X_{P,n},Y_{P,n}),\cdots$ forms a Markov chain with stationary distribution $\pi_P$, we have
		\begin{align}
		P(Yf(X)\leq 0)&=\bbE\big[\bone_{Yf(X)\leq 0}\big]\\
		&=\sum_{P \in \calP}\mu_P \bbE_{\pi_P}\big[\bone_{Yf(X)\leq 0}\big]\\
		&=R(f).
		\end{align}	
		Hence, by applying Proposition \ref{akey}  with $\tilf(x,y):=yf(x)\in \calF:=\pm \calH_{\calK}$ $(M=B\kappa)$ and $\varphi(x)=\min(1,(1-x)_+)\geq \bone_{(-\infty,0]}$, from \eqref{eq36}, it holds that 
		\begin{align}
		R(f) &\leq \frac{1}{n}\sum_{i=1}^n \min(1,[1-Y_if(X_i)/\delta]_+)+ \frac{8}{\delta}\calR_n(\calH_{\calK})\nn\\
		&\qquad + \bigg(\sqrt{\frac{1}{2}\ln \frac{2\pi^2}{3\alpha}}+\log \log 2 \delta^{-1}\bigg)\sqrt{\frac{\tau_{\min}}{n}}+B_n\\
		&\leq \hatR_{\delta}(f)+ \frac{8}{\delta}\calR_n(\calH_{\calK})+  \bigg(\sqrt{\frac{1}{2}\ln \frac{2\pi^2}{3\alpha}}\nn\\
		&\qquad +\log \log 2 \delta^{-1}\bigg)\sqrt{\frac{\tau_{\min}}{n}}+B_n \label{amat}
		\end{align} with probability at least $1-\alpha/2$.
		
		From \eqref{amat} and Lemma \ref{asto1} and the union bound, we obtain \eqref{c}, which concludes our proof of Theorem \ref{thm1}.
	\end{proof}
	Next, we introduce a novel bound based on pseudo-dimension of the kernel family. First, recall the following definition of pseudo-dimension \citep{Nathan06}.
	\begin{definition} Let $\calK=\{K: \calX \times \calX \to \bbR\}$ be a kernel family. The class $\calK$ pseudo-shatters a set of $n$ pairs of points $(\hatX_1,\tilX_1),(\hatX_2,\tilX_2),\cdots, (\hatX_n,\tilX_n)$ if there exists thresholds $t_1,t_2,\cdots, t_n \in \bbR$ such that for any $b_1,b_2,\cdots, b_n \in \{\pm 1\}$ there exists $K \in \calK$ with $\sgn((K(\hatX_i,\tilX_i))-t_i)=b_i$. The pseudo-dimension $d_{\calK}$ is the largest $n$ such that there exists a set of $n$ pairs of points that are pseudo-shattered by $\calK$.
	\end{definition}
	The pseudo-dimension of some class of kernel functions can be upper bounded (cf.~\citep{Nathan06}). For example, consider a family of Gaussian kernels:
	\begin{align}
	\calK_{\rmG}^l&:=\bigg\{K_A: (X_1,X_2)\mapsto e^{-(X_1-X_2)^T A (X_1-X_2)}:\nn\\
	& \qquad A \in \bbR^{l\times l}, A \succeq 0  \bigg\}.
	\end{align} 
	Then, it is known that \citep{Nathan06}:
	\begin{align}
	d_{\calK}(\calK_{\rmG}^l)\leq l(l+1)/2.
	\end{align}
	
	Now, recall the following result.
	\begin{lemma}\citep[Theorem 6]{YimingCOLT2009} \label{coyim} There exists a universal constant $C$ such that, for any $\bx=\{x_i: i \in [n]\}$, there holds:
		\begin{align}
		\calU_n(\calK)\leq C(1+\kappa)^2 d_{\calK}(\log (2e n^2)),
		\end{align}
		where
		\begin{align}
		\calU_n(\calK):=\frac{1}{n}\bbE_{\eps}\bigg[\sum_{i,j \in [n]: i<j} \eps_i \eps_j K(X_i,X_j)\bigg]
		\end{align} is defined as the empirical chaos complexity over $\calK$ (see \citep{YimingCOLT2009} for more details).
	\end{lemma}
	Then, the following can be proved.
	\begin{theorem} \label{thm2} Let $\alpha\in (0,1)$. Then, with probability at least $1-\alpha$, for any $f \in \calH_{\calK}$ and
		$m\geq 1$, 
		\begin{align}
		\calE_{\delta}(f)&\leq \frac{8B}{\delta}\bigg( \sqrt{\frac{C(1+\kappa)^2 d_{\calK}(\log (2e n^2))}{n}}   +\frac{ \kappa}{\sqrt{n}}\bigg)\nn\\
		&\qquad +  \bigg(\sqrt{\frac{1}{2}\ln \frac{2\pi^2}{3\alpha}}+\log \log 2 \delta^{-1}\bigg)\sqrt{\frac{\tau_{\min}}{n}}+B_n  \label{d},
		\end{align}  for any  $\delta \in (0,1]$, where $C$ is a constant defined in Lemma \ref{coyim}.
	\end{theorem}
	The proof of Theorem \ref{thm2} is provided in Supplement Materials (cf. Section \ref{thm:thm2proof}).
	\begin{remark} For i.i.d. dataset, the following uniform convergence result for a class of real-valued functions $\calF$ hold \citep[Theorem 10.1]{Anthony1999b}:
		\begin{align}
		&\bbP\bigg[\calE_{\delta}(f)\geq \eps \enspace \mbox{for some} \enspace f \in \calF\bigg]\nn\\
		&\qquad \leq 2 \calN_{\infty}(\gamma/2,\calF,2n)\exp\bigg(-\frac{\eps^2n}{8}\bigg) \label{amote}
		\end{align} where $\calN_{\infty}(\gamma/2,\calF,2n)$ is the $L_{\infty}$-covering number of the predictor class $\calF$ by considering all possible inputs $\bx$ of size $m$. By using this fact, \citep{Srebro2006LearningBF} provided an margin bound on $\sup_{f \in \calF} \calE_{\delta}(f)$ by using the pseudo-dimension $d_{\calK}$. Later, \citep{Hussain2011} improved this bound for sparse MKL. Unfortunately, for the mixed dataset (or non-i.i.d. dataset in general), \eqref{amote} does not hold since this bound is derived based on some symmetrization properties (permutations) which only hold for i.i.d. datasets. See a detailed proof for \eqref{amote} in \citep{Anthony1999b} with permutation arguments.
	\end{remark}
	By using combinatorial analysis, \citep{Cortes2010} provided the following bounds for the Rademacher complexity $\calR(\calH_{\calK_m^q})$:
	\begin{lemma}\citep[Theorem 4 and Theorem 2]{Cortes2010} \label{corelem} Let $q,r \geq 1$ with $\frac{1}{q}+\frac{1}{r}=1$ and assume that $r$ is an integer. Let $m>1$ and assume that $K(x,x) \leq \kappa^2$ for all $x \in \calX$ and $K \in \calK_m^q$. Then, for any sample $\bz$ of size $n$, the Rademacher complexity of the hypothesis set $\calH_{\calK_m^q}$ can be bounded as follows:
		\begin{align}
		\calR(\calH_{\calK_m^q})\leq B\kappa\sqrt{\frac{\eta_0 r m^{1/r} }{n}}
		\end{align}
		where $\eta_0=23/22$. Especially, it holds that
		\begin{align}
		\calR(\calH_{\calK_m^1})\leq B \kappa\sqrt{\frac{\eta_0 e \lceil \log m \rceil }{n}}.
		\end{align}
	\end{lemma}
	Hence, the following theorem and corollary are direct applications of Lemma \ref{corelem} and Proposition \ref{akey}.
	\begin{theorem} \label{thm3}  Let $q,r \geq 1$ with $\frac{1}{q}+\frac{1}{r}=1$ and assume that $r$ is an integer. Then, for any $\alpha\in (0,1)$ and $m>1$, with probability at least $1-\alpha$, for any $f \in \calH_{\calK_m^q}$, 
		\begin{align}
		&\calE_{\delta}(f)\leq \frac{8B\kappa}{\delta}\sqrt{\frac{\eta_0 r m^{1/r}}{n}} \nn\\
		&\qquad	+  \bigg(\sqrt{\frac{1}{2}\ln \frac{2\pi^2}{3\alpha}}+\log \log 2 \delta^{-1}\bigg)\sqrt{\frac{\tau_{\min}}{n}}+B_n \label{fu}
		\end{align} for any $\delta \in (0,1]$. 
	\end{theorem}
	\begin{corollary} Let $\alpha\in [0,1]$. Then, with probability at least $1-\alpha$, for any $f \in \calH_{\calK_m^1}$,
		\begin{align}
		&\calE_{\delta}(f) \leq \frac{8B\kappa}{\delta}\sqrt{\frac{\eta_0 e\lceil \log m \rceil }{n}} \nn\\
		&\enspace +  \bigg(\sqrt{\frac{1}{2}\ln \frac{\pi^2}{3\alpha}}+\log \log 2 \delta^{-1}\bigg)\sqrt{\frac{\tau_{\min}}{n}}+B_n \label{e}
		\end{align} for any $\delta \in (0,1]$.
	\end{corollary}
	\appendix
	\section{Proof of Lemma \ref{lem:keymodx}}\label{lem:keymodxproof}
	First, recall the following result which was developed base on the spectral method \citep{Pascal2001}:
	\begin{lemma} \citep[Theorems 3.41]{Rudolf2011}  \label{lem:berryessen} Let $X_1,X_2,\cdots,X_n$ be a stationary Markov chain on some Polish space with $L_2$ spectral gap $\lambda$ defined in Section \ref{sec:background} and the initial distribution $\nu \in \calM_2$. Let $f \in \calF$ and define
		\begin{align}
		S_{n,n_0}(f)=\frac{1}{n}\sum_{j=1}^n f(X_{j+n_0})
		\end{align} for all $n_0\geq 0$. Then, it holds that
		\begin{align}
		&\bbE\bigg[\bigg|S_{n,n_0}(f)-\bbE_{\pi}[f(\bX)]\bigg|^2\bigg]\nn\\
		&\qquad \leq \frac{2M}{n(1-\lambda)}+ \frac{64 M^2}{n^2(1-\lambda)^2}\lambda^{n_0}\bigg\|\frac{dv}{d\pi}-1\bigg\|_2.
		\end{align}
	\end{lemma}
	We also recall the following important lemma. 
	\begin{lemma}\cite[Lemma 19]{Truong2022GEB}  \label{truonglem} Let $\{X_n\}_{n=1}^{\infty}$ be  an arbitrary process on a Polish space $\calS$, and let $\{Y_n\}_{n=1}^{\infty}$ be a independent copy (replica) of $\{X_n\}_{n=1}^{\infty}$. Denote by $\bX=(X_1,X_2,\cdots,X_n), \bY=(Y_1,Y_2,\cdots,Y_n)$, and $\calF$ a class of uniformly bounded functions from  $\calS \to \bbR$. Let $\bepsilon:=(\eps_1,\eps_2,\cdots,\eps_n)$ be a vector of i.i.d. Rademacher's random variables. Then, the following holds:
		\begin{align}
		&\bbE_{\bepsilon }\bigg[\bbE_{\bX,\bY} \bigg[\sup_{f \in \calF} \bigg|\sum_{i=1}^n\eps_i (f(X_i)-f(Y_i))\bigg|\bigg]\bigg]\nn\\
		&\qquad =\bbE_{\bX,\bY}\bigg[\sup_{f \in \calF} \bigg|\sum_{i=1}^n (f(X_i)-f(Y_i))\bigg|\bigg]\label{facs}.
		\end{align} 
		In addition, for any $f \in \calF$,
		\begin{align}
		&\bbE_{\bepsilon }\bigg[\bbE_{\bX} \bigg[\bigg|\bbE_{\bY} \bigg[\sum_{i=1}^n\eps_i (f(X_i)-f(Y_i))\bigg]\bigg|\bigg]\bigg]\nn\\
		&\qquad =\bbE_{\bX}\bigg[ \bigg|\bbE_{\bY}\bigg[\sum_{i=1}^n (f(X_i)-f(Y_i))\bigg]\bigg|\bigg]\label{facsto}.
		\end{align} 
	\end{lemma}
	
	Now, we return to the proof of Lemma \ref{lem:keymodx}. For each $f\in \calF$, observe that
	\begin{align}
	&\frac{1}{n}\sum_{i=1}^n f(X_i)-\int_{\calS}  \pi(x) f(x)dx\nn\\
	&= \frac{1}{n}\sum_{i=1}^n f(X_i)- \bbE[f(X_i)]\nn\\
	&\qquad  + \frac{1}{n}\sum_{i=1}^n \bbE[f(X_i)] -\int_{\calS} \pi(x) f(x) dx \label{C1}.
	\end{align}	
	On the other hand, we have
	\begin{align}
	&\bigg|\frac{1}{n}\sum_{i=1}^n \bbE[f(X_i)]-\int_{\calS}  \pi(x) f(x)dx\bigg|\nn\\
	&\quad  =\bigg|\frac{1}{n}\sum_{P\in \calP} \sum_{i\in \calT_P} \bbE[f(X_i)]-\sum_{P \in\ \calP} \mu_P \bbE_{\pi_P}[f(X)]\bigg|\\
	&\quad  \leq \sum_{P\in \calP} \bigg|\frac{1}{n} \sum_{i\in \calT_P} \bbE[f(X_i)]- \mu_P \bbE_{\pi_P}[f(X)]\bigg|\\
	&\quad \leq \sum_{P\in \calP} \bigg|\frac{\mu_P }{T_P} \sum_{i\in \calT_P} \bbE[f(X_i)]- \mu_P \bbE_{\pi_P}[f(X)]\bigg|\\
	&\quad =\sum_{P\in \calP}\mu_P \bigg|\frac{1}{T_P} \sum_{i\in \calT_P} \bbE[f(X_i)]-  \bbE_{\pi_P}[f(X)]\bigg|\\
	&\quad \leq \sum_{P\in \calP}\mu_P \sqrt{\frac{2M}{n(1-\lambda_P)}+ \frac{64 M^2}{n^2(1-\lambda_P)^2}\bigg\|\frac{dv}{d\pi_P}-1\bigg\|_2} \label{keyg0} \\
	&\quad \leq A_n \label{keyg},
	\end{align} where \eqref{keyg0} follows from Lemma \ref{lem:berryessen} with $n_0=0$.
	
	By using $|a+b|\leq |a|+|b$, from \eqref{C1} and \eqref{keyg}, we obtain
	\begin{align}
	&\bbE\big[\big\|P_n-P\big\|_{\calF}\big]\nn\\
	&\qquad \leq \bbE\bigg[\sup_{f\in \calF}\bigg|\frac{1}{n}\sum_{i=1}^n f(X_i)- \bbE[f(X_i)]\bigg|\bigg] + A_n\label{eq15}.
	\end{align} 
	On the other hand, let $Y_1,Y_2,\cdots, Y_n$ is a replica of $X_1,X_2,\cdots, X_n$. It holds that
	\begin{align}
	&\bbE\bigg[\sup_{f\in \calF}\bigg|\frac{1}{n}\sum_{i=1}^n f(X_i)- \bbE[f(X_i)]\bigg|\bigg]\nn\\
	&\qquad =\bbE_{\bX}\bigg[\sup_{f\in \calF}\bigg|\bbE_{\bY}\bigg[\frac{1}{n}\sum_{i=1}^n f(X_i)- f(Y_i)\bigg]\bigg|\bigg] \label{a0} \\
	&\qquad \leq \bbE_{\bX} \bigg[\bbE_{\bY}\bigg[\sup_{f\in \calF}\bigg|\frac{1}{n}\sum_{i=1}^n f(X_i)- f(Y_i)\bigg|\bigg]\bigg] \label{a1}.
	\end{align}
	
	Now, by Lemma \ref{truonglem}, we have
	\begin{align}
	&\bbE\bigg[\sup_{f\in \calF}\bigg|\frac{1}{n}\sum_{i=1}^n f(X_i)-f(Y_i)\bigg|\bigg]\nn\\
	&\qquad =\bbE_{\eps} \bbE_{\bX,\bY}\bigg[\bigg\|\frac{1}{n}\sum_{i=1}^n \eps_i\big(f(X_i)- f(Y_i)\big)\bigg\|_{\calF}\bigg]\\
	&\qquad \leq 2\bbE\big[\|P_n^0\|_{\calF}\big]
	\label{a4},
	\end{align} where \eqref{a4} follows from the fact that $\bY$ is a replica of $\bX$ and the triangle inequality for infinity norm.
	
	From \eqref{eq15} and \eqref{a4}, we finally obtain
	\begin{align}
	\bbE\big[\big\|P_n-P\big\|_{\calF}\big]  \leq 2\bbE\big[\|P_n^0\|_{\calF}\big] +A_n\label{xpess}.
	\end{align} 
	\bibliographystyle{plainnat}
	\bibliography{isitbib}
	\clearpage
	\newpage
	\onecolumn
	\section{Proof of Theorem \ref{thm2}} \label{thm:thm2proof}
	For each fixed $X^n \in \calX^n$ and $f \in \calH_K$ for some $K \in \calK$, observe that
	\begin{align}
	\bigg|\sum_{i=1}^n \eps_i f(X_i)\bigg|&=\bigg|\bigg\langle \bw, \sum_{i=1}^n \eps_i \bPhi_K(X_i)\bigg\rangle_K\bigg|\\
	&\leq \|\bw\|_K \bigg\|\sum_{i=1}^n \eps_i \bPhi_K(X_i)\bigg\|_K\\
	&\leq B \bigg\|\sum_{i=1}^n \eps_i \bPhi_K(X_i)\bigg\|_K\\
	&\leq B \sqrt{\sum_{i,j} \eps_i \eps_j K(X_i,X_j)} \label{mod1}.
	\end{align}
	It follows from \eqref{mod1} that
	\begin{align}
	\bbE_{\eps}\bigg[\sup_{f \in \calH_{\calK}}\bigg|\sum_{i=1}^n \eps_i f(X_i)\bigg|  \bigg]&=\bbE_{\eps}\bigg[\sup_{K\in \calK} \sup_{f \in \calH_K}\bigg|\sum_{i=1}^n\eps_i f(X_i)\bigg|\bigg]\\
	&\qquad \leq B \bbE_{\eps}\bigg[\sup_{K \in \calK}\sqrt{\sum_{i,j} \eps_i \eps_j K(X_i,X_j)}\bigg] \\
	&\qquad \leq B \sqrt{\bbE_{\eps}\bigg[\sup_{K \in \calK}\sum_{i,j} \eps_i \eps_j K(X_i,X_j)\bigg]}\\
	&\qquad =B \sqrt{n\calU_n(\calK)+ \tr(K)} \\
	&\qquad \leq B \sqrt{n\calU_n(\calK)} + B \sqrt{\tr(K)}\\
	&\qquad \leq B \sqrt{n\calU_n(\calK)} + B \kappa \sqrt{n}\\
	&\qquad \leq B \sqrt{Cn(1+\kappa)^2 d_{\calK}(\log (2e n^2))} +B \kappa \sqrt{n}
	\label{mod2},
	\end{align} where \eqref{mod2} follows from Lemma \ref{coyim}.
	
	From \eqref{mod1} and \eqref{mod2}, we obtain
	\begin{align}
	\calR_n(\calH_{\calK})\leq B \sqrt{\frac{C(1+\kappa)^2 d_{\calK}(\log (2e n^2))}{n}}  +\frac{ B \kappa}{\sqrt{n}} \label{metaga}.
	\end{align}
	By using Proposition \ref{akey} with $\calF=\calH_{\calK}$ and \eqref{metaga}, we obtain \eqref{d}. This concludes our proof of Theorem \ref{thm2}.
	\section{Proof of Proposition \ref{akey}}\label{akeyproof}	
	\begin{proof}[Proof of Proposition \ref{akey}]
		Without loss of generality, we can assume that each $\varphi \in \Phi$ takes its values in $[0,1]$ (otherwise, it can be redefined as $\varphi \wedge 1$). Then, it is clear that $\varphi(x)=1$ for $x\leq 0$. Hence, for each fixed $\varphi \in \Phi$ and $f \in \calF$, we obtain
		\begin{align}
		P\{f\leq 0\}&\leq P\varphi(f)\\
		&\leq P_n\varphi(f)+ \|P_n-P\|_{\calG_{\varphi}} \label{xaku0}, 
		\end{align} where
		\begin{align}
		\calG_{\varphi}:=\big\{\varphi\circ f: f \in \calF\big\}.
		\end{align}
		Now, let 
		\begin{align}
		g(\bx):=\sup_{f\in  \calG_{\varphi}}\bigg|\frac{1}{n}\sum_{i=1}^n f(x_i)-Pf\bigg|
		\end{align} for all $\bx=(x_1,x_2,\cdots,x_n) \in \Lambda$.
		
		Then, for all $\bx,\by \in \Lambda \times \Lambda$, we have
		\begin{align}
		\big|g(\bx)-g(\by)\big|	&=\bigg|\sup_{f\in  \calG_{\varphi}}\bigg|\frac{1}{n}\sum_{i=1}^n f(x_i)-Pf\bigg|-\sup_{f\in  \calG_{\varphi}}\bigg|\frac{1}{n}\sum_{i=1}^n f(y_i)-Pf\bigg|\bigg|\\
		&\leq \sup_{f\in  \calG_{\varphi}}\bigg|\bigg|\frac{1}{n}\sum_{i=1}^n f(x_i)-Pf\bigg|-\bigg|\frac{1}{n}\sum_{i=1}^n f(y_i)-Pf\bigg|\bigg|\\
		&\leq  \sup_{f\in  \calG_{\varphi}}\bigg|\frac{1}{n}\sum_{i=1}^n \big(f(x_i)-f(y_i)\big)\bigg|\\
		&\leq  \sup_{f\in  \calG_{\varphi}}\frac{1}{n}\sum_{i=1}^n \big|f(x_i)-f(y_i)\big|\\
		&\leq \frac{1}{n}\sum_{i=1}^n \bone\{x_i \neq y_i\} \label{su1},
		\end{align} where \eqref{su1} follows from $0\leq f(x_i)\leq 1$ for all $x_i \in \Lambda_i$ and $f \in \calG_{\varphi}$. Hence, for any $t>0$, by Lemma \ref{thm6key}, we have
		\begin{align}
		&\bbP\bigg(\|P_n-P\|_{\calG_{\varphi}} \geq \bbE\big[\|P_n-P\|_{\calG_{\varphi}}\big]+t\sqrt{\frac{\tau_{\min}}{n}}\bigg)\\
		&\qquad =\bbP\bigg(g(\bX) \geq \bbE\big[g(\bX)\big]+t\sqrt{\frac{\tau_{\min}}{n}}\bigg)\\ 
		&\qquad \leq 2\exp(-2t^2). 
		\end{align} 
		Hence, with probability at least $1-2\exp(-2t^2)$ for all $f \in \calF$,
		\begin{align}
		\|P_n-P\|_{\calG_{\varphi}}\leq \bbE\big[\|P_n-P\|_{\calG_{\varphi}}\big]+t\sqrt{\frac{\tau_{\min}}{n}} \label{xaku1}.
		\end{align}
		By combining \eqref{xaku0} and \eqref{xaku1}, we have
		\begin{align}
		P\{f\leq 0\}\leq P_n \varphi(f)+ \bbE[\|P_n-P\|_{\calG_{\varphi}}]+ t\sqrt{\frac{\tau_{\min}}{n}} \label{A2}.
		\end{align}
		Now, by Lemma \ref{lem:keymodx} with $\|f\|_{\infty}\leq 1$ for all $f \in \calG_{\varphi}$, it holds that
		\begin{align}
		\bbE\big[\big\|P_n-P\big\|_{\calG_{\varphi}}\big]&\leq  2\bbE\big[\|P_n^0\big\|_{\calG_{\varphi}} \big] + B_n\\
		&=2\bbE\bigg[\bigg\|n^{-1}\sum_{i=1}^n\eps_i \delta_{X_i}\bigg\|_{\calG_{\varphi}}\bigg]+ B_n
		\label{A3},
		\end{align} where $B_n=A_n|_{M=1}$.
		
		Since $(\varphi-1)/L(\varphi)$ is contractive and $\varphi(0)-1=0$, by using the Talagrand's contraction lemma \citep{LedouxT1991book, Truong2022OnRC}, we obtain 
		\begin{align}
		\bbE_{\eps}\bigg\|n^{-1}\sum_{i=1}^n \eps_i \delta_{X_i}\bigg\|_{\calG_{\varphi}} &\leq 2L(\varphi)\bbE_{\eps}\bigg\|n^{-1}\sum_{i=1}^n \eps_i \delta_{X_i}\bigg\|_{\calF} \\
		&=2L(\varphi) R_n(\calF) \label{A4}.
		\end{align}
		From \eqref{A2}, \eqref{A3}, and \eqref{A4}, with probability $1-2\exp(-2t^2)$,  we have for all $f \in \calF$, we have
		\begin{align}
		P\{f\leq 0\}\leq P_n\varphi(f)+ 4 L(\varphi) R_n(\calF)+ t\sqrt{\frac{\tau_{\min}}{n}}+B_n \label{A5}.
		\end{align}
		Let $\delta_k=2^{-k}$ for all $k\geq 0$.  In addition, set $\Phi=\{\varphi_k:k\geq 1\}$, where
		\begin{align}
		\varphi_k(x):=\begin{cases} \varphi(x/\delta_k), & x\geq 0,\\ \varphi(x/\delta_{k-1}),& x<0\end{cases}.
		\end{align}
		Now, for any $\delta \in (0,1]$, there exists $k$ such that $\delta \in (\delta_k,\delta_{k-1}]$. Hence, if $f(X_i)\geq 0$, it holds that $f(X_i)/\delta_k \geq f(X_i)/\delta$, so we have
		\begin{align}
		\varphi_k(f(X_i))&=\varphi\bigg(\frac{f(X_i)}{\delta_k}\bigg)\\
		&\leq \varphi\bigg(\frac{f(X_i)}{\delta}\bigg) \label{b1},
		\end{align} where \eqref{b1} follows from the fact that $\varphi(\cdot)$ is non-increasing.
		
		On the other hand, if $f(X_i)<0$, then $f(X_i)/\delta_{k-1} \geq f(X_i)/\delta$. Hence, we have
		\begin{align}
		\varphi_k(f(X_i))&=\varphi\bigg(\frac{f(X_i)}{\delta_{k-1}}\bigg)\\
		&\leq \varphi\bigg(\frac{f(X_i)}{\delta}\bigg) \label{b2},
		\end{align} where \eqref{b1} follows from the fact that $\varphi(\cdot)$ is non-increasing. 
		
		From \eqref{b1} and \eqref{b2}, we have
		\begin{align}
		P_n\varphi_k\big(f\big)&=\frac{1}{n}\sum_{i=1}^n \varphi_k(f(X_i))\\
		&\leq \frac{1}{n}\sum_{i=1}^n \varphi\bigg(\frac{f(X_i)}{\delta}\bigg)\\
		&= P_n\varphi\bigg(\frac{f}{\delta}\bigg) \label{b3}. 
		\end{align}
		Moreover, we also have
		\begin{align}
		\frac{1}{\delta_k}&\leq \frac{2}{\delta},
		\end{align}
		and
		\begin{align}
		\log k=\log \log_2 \frac{1}{\delta_k}\leq \log \log_2 2 \delta^{-1} \label{atm}.
		\end{align}
		Furthermore, observe that
		\begin{align}
		L(\varphi_k)&=\sup_{x\in \bbR}\bigg|\frac{d \varphi_k(x)}{dx}\bigg|\\
		&=\sup_{x\in \bbR}\bigg|\frac{d \varphi(x/\delta_k)}{dx}\bigg|\bone\{x\geq 0\}\nn\\
		&\qquad +\bigg|\frac{d \varphi(x/\delta_{k-1})}{dx}\bigg|\bone\{x< 0\}\\
		&\leq \frac{L(\varphi)}{\min\{\delta_k,\delta_{k-1}\}} \\
		&=\frac{ L(\varphi)}{\delta_k}\\
		&\leq \frac{2}{\delta} L(\varphi) \label{atm2}.
		\end{align}
		
		Now, by using \eqref{A5} with $\varphi=\varphi_k$ and $t$ is replaced by $t+\sqrt{\log k}$ and then the union bound, we obtain
		\begin{align}
		&\bbP\bigg(\exists f\in \calF: P\{f\leq 0\}> \inf_{k>0}\bigg[P_n\varphi_k(f)+ 4 L(\varphi_k) R_n(\calF) + \bigg(t+\sqrt{\log k}\bigg)\sqrt{\frac{\tau_{\min}}{n}}+ B_n\bigg]\bigg)\nn\\
		&\qquad \leq 2\sum_{k=1}^{\infty}\exp\big(-2\big(t+\sqrt{\log k}\big)^2\big)\\
		&\qquad \leq 2 \sum_{k=1}^{\infty} k^{-2}\exp\big(-2t^2\big)\\
		&\qquad=\frac{\pi^2}{3}\exp\big(-2t^2\big) \label{asp},
		\end{align} where \eqref{asp} follows from
		\begin{align}
		\frac{\pi^2}{6}=\sum_{k=1}^{\infty} k^{-2}.
		\end{align}
		By combining \eqref{asp} with \eqref{b3}, \eqref{atm}, and \eqref{atm2}, for any $t>0$, we obtain
		\begin{align}
		&\bbP\bigg(\exists f \in \calF: P\{f\leq 0\}> \inf_{\delta \in (0,1]}\bigg[P_n\varphi\bigg(\frac{f}{\delta}\bigg) +\frac{8L(\varphi)}{\delta}R_n(\calF)\nn\\
		&\qquad + \bigg(t+\log \log 2 \delta^{-1}\bigg)\sqrt{\frac{\tau_{\min}}{n}}+B_n\bigg]\bigg) \leq \frac{\pi^2}{3}\exp\big(-2t^2\big).
		\end{align}
		Finally, for any $\alpha \in (0,1]$, by choosing $t:=\sqrt{\frac{1}{2}\ln \frac{\pi^2}{3\alpha}}$, we obtain the following $\alpha$-PAC bound:
		\begin{align}
		P\{f\leq 0\}\leq \inf_{\delta \in (0,1]}\bigg[P_n\varphi\bigg(\frac{f}{\delta}\bigg)+\frac{8L(\varphi)}{\delta}R_n(\calF)+ \bigg(\sqrt{\frac{1}{2}\ln \frac{\pi^2}{3\alpha}}+\log \log 2 \delta^{-1}\bigg)\sqrt{\frac{\tau_{\min}}{n}}+B_n\bigg].
		\end{align}
	\end{proof}
	\section{A New Bernstein inequality for the mixed Markov chains} \label{beistein:sm}
	\subsection{A New Beinstein Inequality} 
	Before stating our main result, we introduce some new concepts in Markov chains. 
	\begin{definition} \label{def29} Let $\{X_n\}_{n=1}^{\infty}$ be a sequence of random variables on $\Lambda:=\Lambda_1 \times  \Lambda_2  \times  \cdots \times  \Lambda_n \times \cdots $. A Markov segment of this sequence is defined as a set of random variables $\{X_t,X_{t+1},\cdots, X_{t+K}\}$ such that the probability transition matrices $P_{i-1,i}$ are the same for all $i: t+1\leq i\leq t+K+1$ and $P_{t-1,t}\neq P_{t,t+1}, P_{t+K,t+K+1}\neq P_{t+K-1,t+K}$. Here,
		\begin{align}
		P_{i-1,i}(x,x'):=\bbP\big[X_i=x\big|X_{i-1}=x'\big]
		\end{align} and $(x,x') \in \Lambda_{i-1} \times \Lambda_i, \enspace \forall  i \in [n]$. The set of all different transition probability matrices is defined as $\calP$. The cardinality of $\calP$ is defined as $|\calP|$. 
	\end{definition}

	In this section, we introduce a modified version of Bernstein inequality for the mixed sequence $X^n$ based on the aggregated pseudo spectral gap, which extends the Bernstein inequality for Markov chain in \citep{Daniel2015}.
	\begin{definition} \label{keydef} Let $X_1,X_2,\cdots,X_n$ be a sequence of random variables on $\Lambda_1 \times  \Lambda_2  \times  \cdots \times  \Lambda_n$ with the transition probability sequence $\{P_{i-1,i}(\cdot,\cdot), i \in [n]\}$ defined in Definition \ref{def29}. 	
		Define $n$ linear operators $\bP_{i-1,i}f(x):=\int_{\Lambda_i} P_{i-1,i}(x,y) f(y) dy$ for all $i \in [n]$. Let $|\calP|$ be the number of different Markov chain segments with probability transition matrices on the set $\calP:=\{P_{i-1,i}: i \in [n]\}$. The \emph{aggregated pseudo spectral gap} of the random sequence $X_1,X_2,\cdots,X_n$ is define as
		\begin{align}
		\gamma_{\rm{aps}}=\min_{P \in \calP} \big(\max_{k\geq 1} \big\{\gamma((\bP^*)^k\bP^k)/k\big\}\big),
		\end{align} where $\bP$ is the linear operator associated with the transition matrix $P$\footnote{This definition generalize the pseudo-spectral gap concept in \citep{Daniel2015}.}.
	\end{definition}
	For a mixed dataset of $|\calP|$ Markov chains (data types), $\gamma_{\rm{aps}}=\min\{\gamma_{\rm{ps,P}}: P  \in \calP\}$ where $\gamma_{\rm{ps,P}}$ is the pseudo spectral gap of the $P$-th Markov chain. In general, $\gamma_{\rm{aps}}$ can be estimated based on the same method to estimate $\gamma_{\rm{ps}}$ for the order-$1$ Markov chain, and its value depend on the mixing times of these Markov chains (cf.~Subsection \ref{D2sup} in the supplement material).
	\begin{lemma} \label{thm6keyb} Let $X_1,X_2,\cdots,X_n$ be a sequence of random variable on $\Lambda_1 \times \Lambda_2  \times \cdots \times \Lambda_n$ with the transition probability sequence $P_{i-1,i}(\cdot,\cdot), i \in [n]$. Assume that $P_{i-1,i} \in \calP$ where $\calP$ is defined in Definition \ref{keydef}. In addition, the probability that each sample $X_n$ is taken from a chain segment $v_P$ is $\mu_P$ for all $P \in \calP$. We also assume that $X_1 \sim \nu$. For each $f \in \calF$ where $\calF$ is a class of uniformly bounded functions, i.e., $\|f\|_{\infty}\leq M<\infty$, denote by $V_f:=\max_{P\in \calP} \var_{\pi_P}(f)$.  Let $S:=\sum_{i=1}^n f(X_i)$, then 
		\begin{align}
		\bbP \bigg(\bigg|\frac{S}{n}-\sum_{P \in \calP} \mu_P \bbE_{\pi_P}[f(X)]\bigg|\geq u\bigg)\leq  2\eta\exp\bigg(-\frac{n u^2  \gamma_{\rm{aps}}}{8 |\calP| V_f\big(1+1/\gamma_{\rm{aps}}\big)+20uC}\bigg) \label{eq166}
		\end{align} for any $u\geq 0$, where\footnote{We assume that $\nu<<\pi_P$, so the corresponding Radon-Nikodym derivative exists.}
		\begin{align}
		\eta:=\max_{P \in \calP} \bigg\|\frac{d \nu}{d\pi_P} \bigg\|_{\infty} \label{defc}.
		\end{align}
	\end{lemma}
	\begin{proof}
		Let
		\begin{align}
		\calT_P=\{i: X_i \enspace \mbox{is taken from the Markov chain}\enspace P\}, \quad \forall P \in \calP.
		\end{align} 
		In addition, let $\gamma_{\rm{ps},P}$ be the pseudo-spectral gap of the Markov chain segment $P$ for each $P \in \calP$. Then, by \citep[Proof of Theorem 3.4]{Daniel2015} with $f_1=f_2=\cdots=f_n=f$, for each $P \in \calP$ we have 
		\begin{align}
		\bbE_{\pi_P}\big[e^{\theta \sum_{i \in \calT_P} (f(X_i)-\bbE_{\pi_P}[f(X_i)])}\big] &\leq \exp\bigg(\frac{2(|\calT_P|+1/\gamma_{\rm{ps,P}})\var_{\pi_P}(f)}{\gamma_{\rm{ps,P}}}\theta^2 \bigg(1-\frac{10\theta}{\gamma_{\rm{ps},P}}\bigg)^{-1}\bigg) \label{betcheb}\\
		&\leq \exp\bigg(\frac{2(|\calT_P|+1/\gamma_{\rm{aps}})V_f}{\gamma_{\rm{aps}}}\theta^2 \bigg(1-\frac{10\theta}{\gamma_{\rm{aps}}}\bigg)^{-1}\bigg) \label{betchec},
		\end{align} where \eqref{betchec} follows from $\gamma_{\rm{aps}}\leq \gamma_{\rm{ps},P}$ for all $P\in \calP$. 
		
		By using the change of measure \citep{Billingsley}, from \eqref{betchec}, we obtain
		\begin{align}
		\bbE_{\nu}\big[e^{\theta \sum_{i \in \calT_P} (f(X_i)-\bbE_{\pi_P}[f(X_i)])}\big] &\leq \bigg\|\frac{d\nu}{d\pi_P}\bigg\|_{\infty}\exp\bigg(\frac{2(|\calT_P|+1/\gamma_{\rm{aps}})V_f}{\gamma_{\rm{aps}}}\theta^2 \bigg(1-\frac{10\theta}{\gamma_{\rm{aps}}}\bigg)^{-1}\bigg) \\
		&\leq \eta \exp\bigg(\frac{2(|\calT_P|+1/\gamma_{\rm{aps}})V_f}{\gamma_{\rm{aps}}}\theta^2 \bigg(1-\frac{10\theta}{\gamma_{\rm{aps}}}\bigg)^{-1}\bigg)\label{betchece}.
		\end{align}
		
		Now, since the function $e^{\theta s}$ is convex in $s$, for any tuple $(\mu_1,\mu_2,\cdots,\mu_{|\calP|}) \in (0,1)^{|\calP|}$ such that $\sum_{P \in \calP} \mu_P=1$, it holds that 
		\begin{align}
		\bbE_{\nu}\big[e^{\theta (S-n\sum_{P \in \calP} \mu_P \bbE_{\pi_P}[f(X)])}\big]&=\bbE_{\nu}\big[e^{\sum_{P\in \calP} \mu_P (\theta/\mu_P)  \sum_{i \in \calT_P} (f(X_i)-\bbE_{\pi_P}[f(X_i)])}\big]\\
		&\leq \sum_{P\in \calP} \mu_P  \bbE_{\nu}\big[e^{(\theta/\mu_P) \sum_{i \in \calT_P} (f(X_i)-\bbE_{\pi_P}[f(X_i)])}\big] \label{bage}\\
		&=\sum_{P\in \calP} \mu_P  \bbE_{\nu}\big[e^{\theta \sum_{i \in \calT_P} \tilf_P(X_i)}\big]  \label{matbe}\\
		&\leq \eta\sum_{P \in \calP} \mu_P \exp\bigg(\frac{2(|\calT_P|+1/\gamma_{\rm{aps}})V_f}{\mu_P^2 \gamma_{\rm{aps}}}\theta^2 \bigg(1-\frac{10\theta}{\gamma_{\rm{aps}}}\bigg)^{-1}\bigg) \label{matbe2}
		\end{align} where \eqref{matbe} follows from \eqref{betcheb} by setting $\tilf_P(x):=(f(x)-\bbE_{\pi_P}[f(X)])/\mu_P$, \eqref{matbe2} follows from \eqref{betchece}.
		
		Now, let 
		\begin{align}
		\beta:= \frac{2V_f}{\gamma_{\rm{aps}}}\theta^2 \bigg(1-\frac{10\theta}{\gamma_{\rm{aps}}}\bigg)^{-1} \label{batcha}.
		\end{align}
		Then, from \eqref{matbe2} and \eqref{batcha}, it holds that
		\begin{align}
		\bbE_{\nu}\big[e^{\theta (S-n\sum_{P\in \calP}\mu_P\bbE_{\pi_P}[f(X)])}\big]\leq \eta \sum_{P \in \calP} \mu_P \exp\bigg(\frac{\beta (|\calT_P|+1/\gamma_{\rm{aps}})}{\mu_P^2} \bigg) \label{mutatt1}.
		\end{align}
		By setting 
		\begin{align}
		\mu_P:=\frac{\sqrt{ |\calT_P|+1/\gamma_{\rm{aps}}}}{\sum_{P \in \calP} \sqrt{|\calT_P|+1/\gamma_{\rm{aps}}} }, \quad P \in \calP,
		\end{align} from \eqref{mutatt1}, we have
		\begin{align}
		\bbE_{\nu}\big[e^{\theta (S-n\sum_{P\in \calP}\mu_P\bbE_{\pi_P}[f(X)])}\big] &\leq \eta \sum_{P \in \calP} \mu_P \exp\bigg(\beta \bigg[\sum_{P \in \calP} \sqrt{|\calT_P|+1/\gamma_{\rm{aps}}} \bigg]^2 \bigg) \label{mot2}\\
		&=  \eta \exp\bigg(\beta \bigg[\sum_{P \in \calP} \sqrt{|\calT_P|+1/\gamma_{\rm{aps}}} \bigg]^2 \bigg)\\
		&\leq \eta \exp\bigg(\beta |\calP| \sum_{P \in \calP} |\calT_P|+1/\gamma_{\rm{aps}}\bigg)\\
		&= \eta \exp\big(\beta |\calP| \big(n+|\calP|/\gamma_{\rm{aps}}\big)\big) \label{mutatt2}\\
		&=\eta \exp\bigg(\frac{2V_f}{\gamma_{\rm{aps}}}\theta^2 \bigg(1-\frac{10\theta}{\gamma_{\rm{aps}}}\bigg)^{-1} |\calP| \big(n+|\calP|/\gamma_{\rm{aps}}\big)\bigg)\\
		&\leq \eta \exp\bigg(\frac{2V_f}{\gamma_{\rm{aps}}}\theta^2 \bigg(1-\frac{10\theta}{\gamma_{\rm{aps}}}\bigg)^{-1}  n|\calP|\big(1+1/\gamma_{\rm{aps}}\big)\bigg) \label{mutat3},
		\end{align} where \eqref{mutat3} follows from $|\calP| \leq n$. 
		
		By using Chernoff's bound, it can be shown that 
		\begin{align}
		\bbP_{\nu}\big(\big|S-n\sum_{P\in \calP}\mu_P\bbE_{\pi_P}[f(X)]\big|\geq t\big) \leq 2\eta \exp\bigg(-\frac{t^2 \gamma_{\rm{aps}}}{8 n|\calP|\big(1+1/\gamma_{\rm{aps}}\big)V_f+20tC}\bigg) \label{mobet}
		\end{align} where \eqref{mobet} follows from \citep{Daniel2015} with $n$ being replaced by $n |\calP|$ by the fact \eqref{mutat3}, where $\eta$ is defined in \eqref{defc}. 
	\end{proof} 
	\subsection{Relationship between the aggregated pseudo spectral gap and the mixing time}\label{D2sup}
	In this section, we provide a relationship between the \emph{aggregated pseudo spectral gap} $\gamma_{\rm{aps}}$ of the sequence $X_1,X_2,\cdots,X_n$ to its mixing time, called \emph{aggregated mixing time}. First, we introduce the aggregated mixing time of a sequence of random variables.
	\begin{definition} \label{keydefmix} Let $X_1,X_2,\cdots,X_n$ be a sequence of random variables on $\Omega_n:=\Lambda_1 \times \Lambda_2  \times \cdots \times \Lambda_n$ with the transition probability sequence $\calP:=P_{i-1,i}(\cdot,\cdot), i \in [n]$, and
		\begin{align}
		P_{i-1,i}(x,x'):= \bbP\big[X_i=x|X_{i-1}=x'\big], \quad \forall x, x' \in \Lambda_{i-1}\times \Lambda_i.
		\end{align}
		Define $n$ linear operators $\bP_{i-1,i}f(x):=\int_{\Lambda_i} P_{i-1,i}(x,y) f(y) dy$ for all $i \in [n]$. Let $\calP:=\{P_{i-1,i}: i \in [n]\}$. Assume that  for any $P \in \calP$, the Markov segment with stochastic matrix $P$ has a stationary distribution $\pi_P$. Furthermore, for any $\eps>0$ and $P \in \calP$, let
		\begin{align}
		t_{\rm{mix},P}(\eps):=\min\big\{t:  \sup_{x\in \Omega} d_{\rm{TV}}\big(P^t(x,\cdot),\pi_P\big) \leq \eps \big\}.
		\end{align} 
		Then, the \emph{aggregated mixing time} of the random sequence $X_1,X_2,\cdots,X_n$ is defined as
		\begin{align}
		t_{\rm{aps}}(\eps)=\max_{P \in \calP} t_{\rm{mix},P}(\eps).
		\end{align} 
	\end{definition}
	From the above facts, the following result can be proved.
	\begin{theorem} \label{mixbound} For any sequence $X_1,X_2,\cdots,X_n$ such that all its Markov segments have the stationary distribution, it holds that
		\begin{align}
		\gamma_{\rm{aps}}\geq \frac{1-2\eps}{t_{\rm{amix}}(\eps)}, \quad \forall 0\leq \eps<1/2.
		\end{align}
		In particular, we have
		\begin{align}
		\gamma_{\rm{aps}}\geq \frac{1}{2t_{\rm{amix}}}
		\end{align} where
		\begin{align}
		t_{\rm{amix}}:=t_{\rm{amix}}(1/4).
		\end{align}
	\end{theorem}
	\begin{proof} The proof is based on \citep[Proof of Proposition 3.4]{Daniel2015}. First, we show that for all $P \in \calP$, it holds that
		\begin{align}
		\gamma\big((\bP^*)^{\tau_{\rm{mix},P}(\eps)} \bP^{\tau_{\rm{mix},P}(\eps)}\big)\geq 1-\eps \label{afact1}.
		\end{align}
		Indeed, let $L_{\infty}(\pi_P)$ be the set of $\pi$-almost surely bounded functions, equipped with the infinity-norm $\|f\|_{\infty}:=\mbox{ess sup}_{x \in \Omega_n} |f(x)|$. Then $L_{\infty}(\pi_P)$ is a Banach space. In addition, for any $P \in \calP$, $\bP^* \bP$  is self-adjoint, bounded linear operator on $L_2(\pi)$. Define the operator $\bpi$ on $L_2(\pi)$ as $\bpi(f)(x):=\pi(f)$. This is also a self-adjoint bounded operator. Let $\bM:=(\bP^*)^{\tau_{\rm{mix},P}(\eps)}\bP^{\tau_{\rm{mix},P}(\eps)}-\bpi$. Then, we can express the absolute spectral gap of $\gamma_P^*$ of $(\bP^*)^{\tau_P(\eps)}\bP^{\tau_P(\eps)}$ as \citep{Daniel2015}:
		\begin{align}
		\gamma_P^*=1-\sup\bigg\{|\lambda|: \lambda \in \calS_2(\bM)\bigg\} \label{ub1}, 
		\end{align} where the spectrum $\calS_2$ is defined in Subsection 2.2 of the main document.  Thus, $1-\gamma^*$ equals to the spectral radius of $\bM$ on $L_2(\pi_P)$. Since $L_{\infty}(\pi_P)$ is a dense subset of the Hilbert space $L_2(\pi_P)$, it holds that $L_{\infty}(\pi_P)$ is also a Banach space. Hence, by Gelfand's theorem, it holds that
		\begin{align}
		\sup\big\{|\lambda|: \lambda \in \calS_2(\bM)\big\}=\lim_{k\to \infty} \|\bM^k\|_{\infty}^{1/k} \label{ub2}.
		\end{align} 
		Hence, from \eqref{ub1} and \eqref{ub2}, we have
		\begin{align}
		1-\gamma_P^*&=\lim_{k\to \infty} \|\bM^k\|_{\infty}^{1/k} \label{bulet}\\
		&\leq \|\bM\|_{\infty}\\
		&=\|(\bP^*)^{t_{\rm{mix},P}(\eps)} \bP^{t_{\rm{mix},P}(\eps)}-\bpi\big\|_{\infty}  \label{S1}.
		\end{align}
		Now, from definition, it is easy to see that
		\begin{align}
		\bP^n f(x)=\int_y P^n(x,y) f(y) dy, \quad \forall n \in \bbZ_+.
		\end{align}
		It follows that
		\begin{align}
		&\|(\bP^*)^{t_{\rm{mix},P}(\eps)}  \bP^{t_{\rm{mix},P}(\eps)}-\bpi\big\|_{\infty}\nn\\
		&\qquad =\mbox{ess sup}_{x \in \Omega_n} \sup_{\|f\|_\infty=1}  \big|\big((\bP^*)^{t_{\rm{mix},P}(\eps)} \bP^{t_{\rm{mix},P}(\eps)}-\bpi \big)f(x)\big|\\
		&\qquad =\mbox{ess sup}_{x \in \Omega_n} \sup_{\|f\|_\infty=1}  \bigg|\int_y \int_{y'} (P^*)^{t_{\rm{mix},P}(\eps)}(x,y')  P^{t_{\rm{mix},P}(\eps)}(y',y)-\int_{y} \pi(y) f(y) dy \bigg|\\
		&\qquad \leq \mbox{ess sup}_{x \in \Omega_n} \int_y \bigg|\int_{y'} \big(P^*\big)^{t_{\rm{mix},P}(\eps)} (x,y')  P^{t_{\rm{mix},P}(\eps)}(y',y)dy'-\pi(y)\bigg| dy \\
		&\qquad = \mbox{ess sup}_{x \in \Omega_n} \int_y \bigg|\int_{y'} (P^*)^{t_{\rm{mix},P}(\eps)}(x,y')  P^{t_{\rm{mix},P}(\eps)}(y',y)dy'-\int_{y'}  \big(P^*\big)^{t_{\rm{mix},P}(\eps)}(x,y') \pi(y)dy'\bigg| dy \\
		&\qquad \leq  \mbox{ess sup}_{x \in \Omega_n} \int_y \int_{y'} (P^*)^{t_{\rm{mix},P}(\eps)}(x,y')  \big|P^{t_{\rm{mix},P}(\eps)}(y',y)- \pi(y)\big|dy' dy \\
		&\qquad \leq  \bigg(\mbox{ess sup}_{y'\in \Omega_n} \int_y \big|P^{t_{\rm{mix},P}(\eps)}(y',y)- \pi(y)\big| dy\bigg) \bigg(\mbox{ess sup}_{x \in \Omega_n} \int_{y'} (P^*)^{t_{\rm{mix},P}(\eps)}(x,y')dy'\bigg)   \\
		&\qquad =\mbox{ess sup}_{y'\in \Omega_n} \int_y \big|P^{t_{\rm{mix},P}(\eps)}(y',y)- \pi(y)\big| dy\\
		&\qquad =2\enspace \mbox{ess sup}_{x \in \Omega_n} d_{\rm{TV}}(P^{t_{\rm{mix},P}(\eps)}(x,\cdot),\pi)\\
		&\qquad \leq 2\eps \label{mubab}.
		\end{align}
		From \eqref{S1} and \eqref{mubab}, we have
		\begin{align}
		\gamma_P^* \geq 1-2\eps, \qquad \forall P \in \calP,
		\end{align} which leads to \eqref{afact1}. It follows that
		\begin{align}
		\gamma_{\rm{ps},P}:=\max_{k\geq 1} \gamma\big((\bP^*)^k \bP^k/k\big)\geq \gamma\big((\bP^*)^{t_{\rm{mix},P}(\eps)} \bP^{t_{\rm{mix},P}(\eps)}/t_{\rm{mix},P}(\eps)\big)\geq \frac{1-2\eps}{t_{\rm{mix},P}(\eps)} \label{afact2}.
		\end{align}
		From the definitions of $t_{\rm{amix}}(\eps)$, $t_{\rm{mix},P}(\eps)$ and \eqref{afact2}, we obtain
		\begin{align}
		\gamma_{\rm{amix}}&=\max_{P \in \calP} \gamma_{\rm{ps},P}\\
		&\geq \max_{P \in \calP} \frac{1-2\eps}{t_{\rm{mix},P}(\eps)}\\
		&= \frac{1-2\eps}{t_{\rm{amix}}(\eps)}.
		\end{align}
	\end{proof} 
	\begin{definition} A Markov chain $(P,\pi)$ is ergodic if $P^k>0$ (entry-wise positive) for some $k\geq 1$, i.e., $P$ is a \emph{primitive} matrix. If $P$ is ergodic, it has a unique stationary distribution $\pi$ and moreover $\pi_* >0$ where $\pi_*=\min_{i \in \calS} \pi_i$ is called the \emph{minimum stationary probability}. It is obvious that an ergodic Markov chain is irreducible and recurrent.   
	\end{definition}
	It is known that the mixing time of an ergodic and reversible Markov chain is controlled by its absolute spectral gap $\gamma^*$ and the minimum stationary probability $\pi_*$:
	\begin{align}
	\bigg(\frac{1}{\gamma^*}-1\bigg)\ln 2 \leq t_{\rm{mix}}\leq \frac{\ln (4/\pi_*)}{\gamma_*}.
	\end{align}  For non-reversible Markov chains, the relationship between the spectrum and the mixing time is not nearly as straightforward. Any complex eigenvalue with $|\lambda|\leq 1$ provides a lower bound on the mixing time 
	\begin{align}
	\bigg(\frac{1}{1-|\lambda|}-1\bigg)\ln 2\leq t_{\rm{mix}},
	\end{align} and upper bounds may be obtained in terms of the spectral gap of the multiplicative reversiblization
	\begin{align}
	t_{\rm{mix}}\leq \frac{2}{\gamma(\bP^* \bP)}\ln \bigg(2 \sqrt{\frac{1-\pi_*}{\pi_*}}\bigg).
	\end{align}
	Unfortunately, the later estimate is far from sharp. In \citep{Daniel2015}, the following bound is given.
	\begin{align}
	\frac{1}{2 \gamma_{\rm{ps}}} \leq t_{\rm{mix}}\leq \frac{1}{\gamma_{\rm{ps}}}\bigg(\ln \frac{1}{\pi_*}+2 \ln 2+1 \bigg).
	\end{align}
	\citep{WK19ALT} showed algorithms to estimate $\gamma_{\rm{ps}}$ and $\tau_{\rm{mix}}$ of an arbitrary ergodic finite-state Markov chain from a single trajectory of length $m$.
\end{document}

%% file: preamble.tex
\usepackage[mathscr]{eucal}
\usepackage[cmex10]{amsmath}
\usepackage{epsfig,epsf,psfrag}
\usepackage{amssymb,amsmath,amsthm,amsfonts,latexsym,bm}
\usepackage{amsmath,graphicx,bm,xcolor,url}
\usepackage[caption=false]{subfig} 
\usepackage{array}
\usepackage{verbatim}
\usepackage{bm}
\usepackage{verbatim}
\usepackage{textcomp}
\usepackage{mathrsfs}
\usepackage{multirow}
\usepackage{epstopdf}

\catcode`~=11 \def\UrlSpecials{\do\~{\kern -.15em\lower .7ex\hbox{~}\kern .04em}} \catcode`~=13 

\allowdisplaybreaks[1]
 
\newcommand{\nn}{\nonumber}

\newcommand{\calH}{\mathcal{H}}
\newcommand{\calI}{\mathcal{I}}

\newcommand{\calK}{\mathcal{K}}
\newcommand{\calL}{\mathcal{L}}
\newcommand{\calN}{\mathcal{N}}

\newcommand{\calP}{\mathcal{P}}
\newcommand{\calR}{\mathcal{R}}
\newcommand{\calT}{\mathcal{T}}
\newcommand{\calU}{\mathcal{U}}
\newcommand{\calY}{\mathcal{Y}}

\newcommand{\bE}{\mathbf{E}}

\newcommand{\bI}{\mathbf{I}}

\newcommand{\bM}{\mathbf{M}}

\newcommand{\bP}{\mathbf{P}}

\newcommand{\bw}{\mathbf{w}}
\newcommand{\bz}{\mathbf{z}}

\newcommand{\rmG}{\mathrm{G}}

\newcommand{\bbC}{\mathbb{C}}

\DeclareMathAlphabet{\mathbsf}{OT1}{cmss}{bx}{n}
\DeclareMathAlphabet{\mathssf}{OT1}{cmss}{m}{sl}

\DeclareSymbolFont{bsfletters}{OT1}{cmss}{bx}{n}  
\DeclareSymbolFont{ssfletters}{OT1}{cmss}{m}{n}
\DeclareMathSymbol{\bsfGamma}{0}{bsfletters}{'000}
\DeclareMathSymbol{\ssfGamma}{0}{ssfletters}{'000}
\DeclareMathSymbol{\bsfDelta}{0}{bsfletters}{'001}
\DeclareMathSymbol{\ssfDelta}{0}{ssfletters}{'001}
\DeclareMathSymbol{\bsfTheta}{0}{bsfletters}{'002}
\DeclareMathSymbol{\ssfTheta}{0}{ssfletters}{'002}
\DeclareMathSymbol{\bsfLambda}{0}{bsfletters}{'003}
\DeclareMathSymbol{\ssfLambda}{0}{ssfletters}{'003}
\DeclareMathSymbol{\bsfXi}{0}{bsfletters}{'004}
\DeclareMathSymbol{\ssfXi}{0}{ssfletters}{'004}
\DeclareMathSymbol{\bsfPi}{0}{bsfletters}{'005}
\DeclareMathSymbol{\ssfPi}{0}{ssfletters}{'005}
\DeclareMathSymbol{\bsfSigma}{0}{bsfletters}{'006}
\DeclareMathSymbol{\ssfSigma}{0}{ssfletters}{'006}
\DeclareMathSymbol{\bsfUpsilon}{0}{bsfletters}{'007}
\DeclareMathSymbol{\ssfUpsilon}{0}{ssfletters}{'007}
\DeclareMathSymbol{\bsfPhi}{0}{bsfletters}{'010}
\DeclareMathSymbol{\ssfPhi}{0}{ssfletters}{'010}
\DeclareMathSymbol{\bsfPsi}{0}{bsfletters}{'011}
\DeclareMathSymbol{\ssfPsi}{0}{ssfletters}{'011}
\DeclareMathSymbol{\bsfOmega}{0}{bsfletters}{'012}
\DeclareMathSymbol{\ssfOmega}{0}{ssfletters}{'012}

\newcommand{\tilf}{\tilde{f}}

\newcommand{\hatP}{\hat{P}}

\newcommand{\hatR}{\hat{R}}

\newcommand{\hatX}{\hat{X}}

\newcommand{\tilX}{\tilde{X}}

\newcommand{\barx}{\bar{x}}

\newcommand{\barX}{\bar{X}}

\newcommand{\bpi}{\bm{\pi}}
\newcommand{\bepsilon}{\bm{\epsilon}}

\newcommand{\eps}{\varepsilon}

\DeclareMathOperator{\diag}{diag}

\DeclareMathOperator{\sgn}{sgn}
\DeclareMathOperator{\tr}{tr}

\DeclareMathOperator{\var}{\mathsf{Var}}

\newcommand{\bone}{\mathbf{1}}

\newtheorem{theorem}{Theorem} 
\newtheorem{lemma}[theorem]{Lemma}

\newtheorem{proposition}[theorem]{Proposition}
\newtheorem{corollary}[theorem]{Corollary}
\newtheorem{definition}[theorem]{Definition}

\newtheorem{remark}[theorem]{Remark}

\newcommand{\qednew}{\nobreak \ifvmode \relax \else
      \ifdim\lastskip<1.5em \hskip-\lastskip
      \hskip1.5em plus0em minus0.5em \fi \nobreak
      \vrule height0.75em width0.5em depth0.25em\fi}